%% file: main.tex
\documentclass{article}
\usepackage{iclr2026_conference,times}
\pdfoutput=1

\input{math_commands.tex}

\usepackage{xcolor}
\newcommand{\gray}[1]{\textcolor{gray}{#1}}
\usepackage[section]{placeins}
\usepackage[table]{xcolor}
\definecolor{paretoBlue}{RGB}{220,236,252}
\newcommand{\bestbg}[1]{\cellcolor{paretoBlue}{#1}}
\usepackage{booktabs,subcaption}
\usepackage{caption}
\usepackage{hyperref}
\usepackage{url}
\usepackage{multirow, booktabs}
\usepackage{wasysym}
\usepackage{graphicx}
\usepackage{amsmath,amsfonts,amssymb,amsthm,amsthm}
\usepackage{subcaption}
\usepackage[capitalize,noabbrev]{cleveref}
\usepackage{enumitem}      
\usepackage{array} 
\usepackage{wrapfig}
\usepackage{graphicx} 
\graphicspath{{figs/}}

\theoremstyle{plain}
\newtheorem{theorem}{Theorem}[section]

\theoremstyle{definition}
\newtheorem{definition}[theorem]{Definition}
\newtheorem{assumption}[theorem]{Assumption}

\theoremstyle{remark}

\usepackage{algorithm}
\usepackage{algorithmic}
\iclrfinalcopy

\title{Randomized Antipodal Search Done Right for Data Pareto Improvement of LLM Unlearning}

\author{
Ziwen Liu\textsuperscript{1}, 
Huawei Lin\textsuperscript{2}, 
Yide Ran\textsuperscript{3}, 
Denghui Zhang\textsuperscript{3}, 
Jianwen Xie\textsuperscript{4}, 
Chuan Li\textsuperscript{4}, \\
\textbf{Weijie Zhao}\textsuperscript{2}, 
\textbf{Zhaozhuo Xu}\textsuperscript{3}\thanks{Corresponding author: \texttt{zxu79@stevens.edu}} \\
\textsuperscript{1}Rice University, \textsuperscript{2}Rochester Institute of Technology, 
\textsuperscript{3}Stevens Institute of Technology, 
\textsuperscript{4}Lambda \\
\texttt{zl166@rice.edu,\{hl3352,wjzvcs\}@rit.edu.}\\
\texttt{\{yran1,dzhang42,zxu79\}@stevens.edu, \{jianwen.xie, c\}@lambdal.com}
}

\begin{document}

\maketitle

\begin{abstract}
Large language models (LLMs) sometimes memorize undesirable knowledge, which must be removed after deployment. Prior work on machine unlearning has focused largely on optimization methods that adjust parameters to enforce forgetting while preserving retention. However, these approaches assume that the forget and retain sets are readily available, which rarely holds in practice. Unlearning is typically triggered by an undesired generation at inference time, making the retrieval of relevant data the central challenge. 
We introduce the notion of \emph{data Pareto improvement} for LLM unlearning, which formalizes how retrieval can expand the achievable trade-off frontier between forgetting and retention. To realize this principle, we propose \emph{Randomized Antipodal Search on Linearized Influence Kernel (RASLIK)}, a retrieval algorithm that combines permutation–projection hashing with randomized antipodal search. RASLIK reduces selection variance, achieves sublinear complexity, and yields a double gain in both quality and efficiency. Across multiple models, datasets, and unlearning algorithms, RASLIK consistently outperforms deterministic baselines and even oracle sampling, establishing randomized search as a principled and scalable solution for data-centric unlearning.
\end{abstract}

\input{intro}

\input{problem}

\input{method}

\input{exp}

\input{related}

\input{conclusion}

\bibliography{ref,xzz}
\bibliographystyle{iclr2026_conference}

\appendix
\input{appendix}

\end{document}

%% file: math_commands.tex
\usepackage{amsmath,amsfonts,bm}

\def\eqref#1{equation~\ref{#1}}

\def\1{\bm{1}}

\DeclareMathAlphabet{\mathsfit}{\encodingdefault}{\sfdefault}{m}{sl}
\SetMathAlphabet{\mathsfit}{bold}{\encodingdefault}{\sfdefault}{bx}{n}



%% file: intro.tex
\section{Introduction}

Large language models (LLMs) have demonstrated impressive capabilities across diverse tasks \citep{openai2024gpt4technicalreport}, but they sometimes memorize undesirable knowledge \citep{carlini2019secretsharerevaluatingtesting,carlini2023quantifyingmemorizationneurallanguage}. When such information must be removed after deployment, \emph{machine unlearning} provides a mechanism to forget targeted knowledge while preserving general utility \citep{eldan2023whosharrypotterapproximate}. Existing work has primarily focused on designing optimizers, such as gradient-based\citep{jang2022knowledgeunlearningmitigatingprivacy,pmlr-v199-liu22a,yao2024machineunlearningpretrainedlarge,yoon2023gradientascentposttrainingenhances} or preference-based methods \citep{zhang2024negativepreferenceoptimizationcatastrophic,NEURIPS2023_a85b405e,maini2024tofutaskfictitiousunlearning,meng2024simposimplepreferenceoptimization}, that couple forgetting objectives with retention regularizers. These approaches are effective under controlled benchmarks \citep{maini2024tofutaskfictitiousunlearning,shi2024musemachineunlearningsixway} but typically assume that the forget and retain sets are readily available \citep{shi2024musemachineunlearningsixway}. In practice, unlearning is triggered by an undesired generation at inference time, leaving practitioners with only the observed output and a massive training corpus. \emph{Identifying which data to forget and which to retain becomes the primary challenge}, making data efficiency the central bottleneck of unlearning \citep{carlini2021extractingtrainingdatalarge}.

Unlearning inherently involves balancing two seemingly conflicting goals: improving forgetting often reduces retention, while prioritizing retention risks incomplete forgetting \citep{10.1145/3603620,nguyen2024surveymachineunlearning}. This trade-off defines a Pareto frontier \citep{DAVTALABOLYAIE2021247} of achievable outcomes. We introduce the concept of \emph{data Pareto improvement} in LLM unlearning, which highlights the role of retrieval in expanding this frontier. A retrieval mechanism is Pareto-improving if it enables stronger forgetting without disproportionate loss of retention, or conversely preserves retention without undermining forgetting. This perspective shifts the focus of unlearning from being purely optimization-centric to being fundamentally retrieval-centric. Retrieval quality is not a preprocessing detail but a first-order determinant of unlearning outcomes.

Building on this insight, we propose \emph{Randomized Antipodal Search on Linearized Influence Kernel (RASLIK)}, a retrieval algorithm that introduces controlled randomization into influence-based search. RASLIK constructs randomized gradient sketches via permutation–projection hashing and performs antipodal search to identify both aligned samples to forget and anti-aligned samples to retain. Randomization smooths unstable thresholding decisions, reducing selection variance, while sketching achieves sublinear complexity. The result is a double gain in both quality and efficiency. Experiments across models, datasets, and unlearning algorithms show that RASLIK consistently shifts the Pareto frontier outward, outperforming deterministic baselines and even oracle sampling.

Our contributions are as follows:
\begin{itemize}[nosep,leftmargin=*]
    \item We identify retrieval as the central bottleneck of practical LLM unlearning and highlight data efficiency as a major challenge beyond optimization design.
    \item We introduce the notion of \emph{data Pareto improvement}, formalizing how retrieval can expand the achievable forgetting–retention frontier in unlearning.
    \item We propose \emph{RASLIK}, a randomized antipodal search method on linearized influence kernels that reduces variance, achieves sublinear retrieval complexity, and enables more stable and effective unlearning.
    \item We validate RASLIK through extensive experiments, demonstrating consistent Pareto improvements across benchmarks, algorithms, and model scales.
\end{itemize}

%% file: problem.tex
\section{Data Pareto Improvement of LLM Unlearning}

\subsection{A Focus on Data Efficiency of LLM Unlearning}

\begin{wrapfigure}{r}{0.5\textwidth}
  \vspace{-1.2em}
  \centering
  \includegraphics[width=\linewidth]{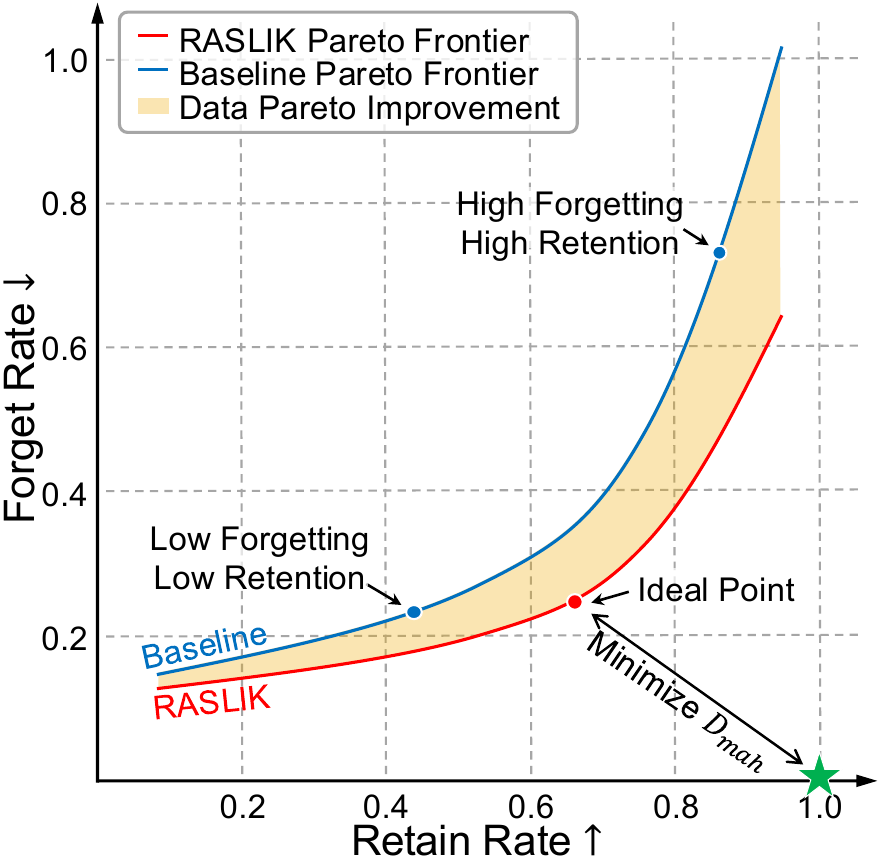}
  \vspace{-1.8em}
  \caption{Pareto trade-off between forgetting and retention in LLM unlearning.}
  \label{fig:pareto}
  \vspace{-1.3em}
\end{wrapfigure}
Large Language Models (LLMs) trained on massive corpora inevitably memorize undesirable knowledge \citep{carlini2019secretsharerevaluatingtesting}. In these cases, model owners must \emph{unlearn} such knowledge while preserving the model’s utility~\citep{carlini2023quantifyingmemorizationneurallanguage}. Formally, given parameters $\theta \in \mathbb{R}^d$ and a loss $\ell(x;\theta)$ for sample $x$, the goal of unlearning is to increase loss on a designated \emph{forget set} $\mathcal{F}$ while maintaining or improving performance on a complementary \emph{retain set} $\mathcal{R}$.
Existing work mostly treats unlearning as an optimization problem: designing loss functions that couple a forgetting objective with a utility-preserving regularizer. Examples include gradient ascent on $\mathcal{F}$ with gradient descent on $\mathcal{R}$ \citep{jang2022knowledgeunlearningmitigatingprivacy,pmlr-v199-liu22a,yao2024machineunlearningpretrainedlarge}. These paradigms implicitly assume that \emph{the forget set $\mathcal{F}$ and the retain set $\mathcal{R}$ are given}. In practice, however, unlearning rarely begins with this setting. Instead, it is triggered by an \emph{unexpected generation} $y$ produced at inference time. Faced with only $y$ and a massive training set, practitioners must first determine \textbf{\emph{what to forget}} and \textbf{\emph{what to retain}}. This makes retrieval of $\mathcal{F}$ and $\mathcal{R}$ not a secondary step but the true bottleneck in practical unlearning. Without high-quality retrieval, even the most sophisticated optimizers cannot achieve effective forgetting.

\subsection{Introducing Data Pareto Improvement Formulation to LLM Unlearning}
Unlearning introduces a fundamental tension: improving the degree of forgetting often reduces the model’s general capabilities, while prioritizing retention risks incomplete forgetting. As shown in Figure~\ref{fig:pareto}, this tension can be formalized as a \emph{Pareto trade-off} between two conflicting objectives:
\[
\text{maximize forgetting accuracy} \quad \text{vs.} \quad \text{maximize retention quality}.
\]
Any unlearning method, therefore, lies on a Pareto frontier \citep{DAVTALABOLYAIE2021247}: improvements in one dimension typically come at a cost in the other. Unlike ordinary optimization, where one seeks a single optimum, unlearning inherently requires balancing two competing goals.

This motivates a \emph{data-centric} notion of Pareto efficiency. We define \textbf{data Pareto efficiency} as the ability of the retrieval stage to identify $\mathcal{F}$ and $\mathcal{R}$ that \emph{shift the Pareto frontier outward}. Concretely, a data selection is Pareto-improving if it enables one of the following without degrading the other:
\begin{itemize}[nosep,leftmargin=*]
    \item Achieving stronger forgetting (the model reliably suppresses $y$ and its variants) without disproportionate loss of retention.
    \item Preserving or enhancing retention (general capabilities remain intact) without sacrificing forgetting performance.
\end{itemize}

Seen this way, retrieval quality is not a preprocessing detail but a first-order determinant of unlearning outcomes. A retrieval mechanism explicitly designed to respect the Pareto structure can systematically enable better trade-offs for downstream optimizers. We therefore introduce the concept of \emph{data Pareto improvement}: improvements in the selection of $\mathcal{F}$ and $\mathcal{R}$ that expand the achievable frontier of forgetting–retention performance. This perspective reframes unlearning from being solely \emph{optimization-centric} to being also fundamentally \emph{retrieval-centric}.

%% file: method.tex

\section{Randomized Antipodal Search on Linearized Influence Kernel}

\noindent \textbf{Notations.}  
Let $\theta \in \mathbb{R}^d$ denote the model parameters, $\ell(x;\theta)$ the loss for input $x$ in training dataset $X$, and $g(x;\theta) = \nabla_\theta \ell(x;\theta)$ its gradient.  
For a target generation $y$, define $q_y = g(y;\theta)$. For a training item $x \in X$, write $g_x = g(x;\theta)$.  
The unlearning objective is
\[
U(\theta) = \mathbb{E}_{x \in \mathcal{F}}[\ell(x;\theta)] - \mathbb{E}_{x \in \mathcal{R}}[\ell(x;\theta)], \quad \nabla_\theta U(\theta) = \frac{1}{|\mathcal{F}|}\sum_{x \in \mathcal{F}} g(x;\theta) - \frac{1}{|\mathcal{R}|}\sum_{x \in \mathcal{R}} g(x;\theta),
\]
where $\nabla_\theta U(\theta)$ denotes the combined gradient computed from both forget and retain sets. This formulation is defined as Gradient Ascent with Gradient Descent on the Retain set (GA\_GDR) \citep{jang2022knowledgeunlearningmitigatingprivacy,pmlr-v199-liu22a}.
Moreover, we define the update direction of $\theta$ as
\begin{align}\label{eq:direction}
\Delta(\mathcal{F},\mathcal{R}) = -\nabla_\theta U(\theta)
= \frac{1}{|\mathcal{R}|}\sum_{x\in\mathcal{R}} g_x
  - \frac{1}{|\mathcal{F}|}\sum_{x\in\mathcal{F}} g_x,
\end{align}
where the forget set $\mathcal{F}$ aligns with $q_y$ and the retain set $\mathcal{R}$ anti-aligns with $q_y$. In this work, our goal is to retrieve both sets given $q_y$.  

\subsection{Random Linearization of Influence Kernel via Permute-Project Hashing}  
\begin{wrapfigure}{r}{0.48\textwidth}
  \vspace{-1.2em}
  \centering
  \includegraphics[width=\linewidth]{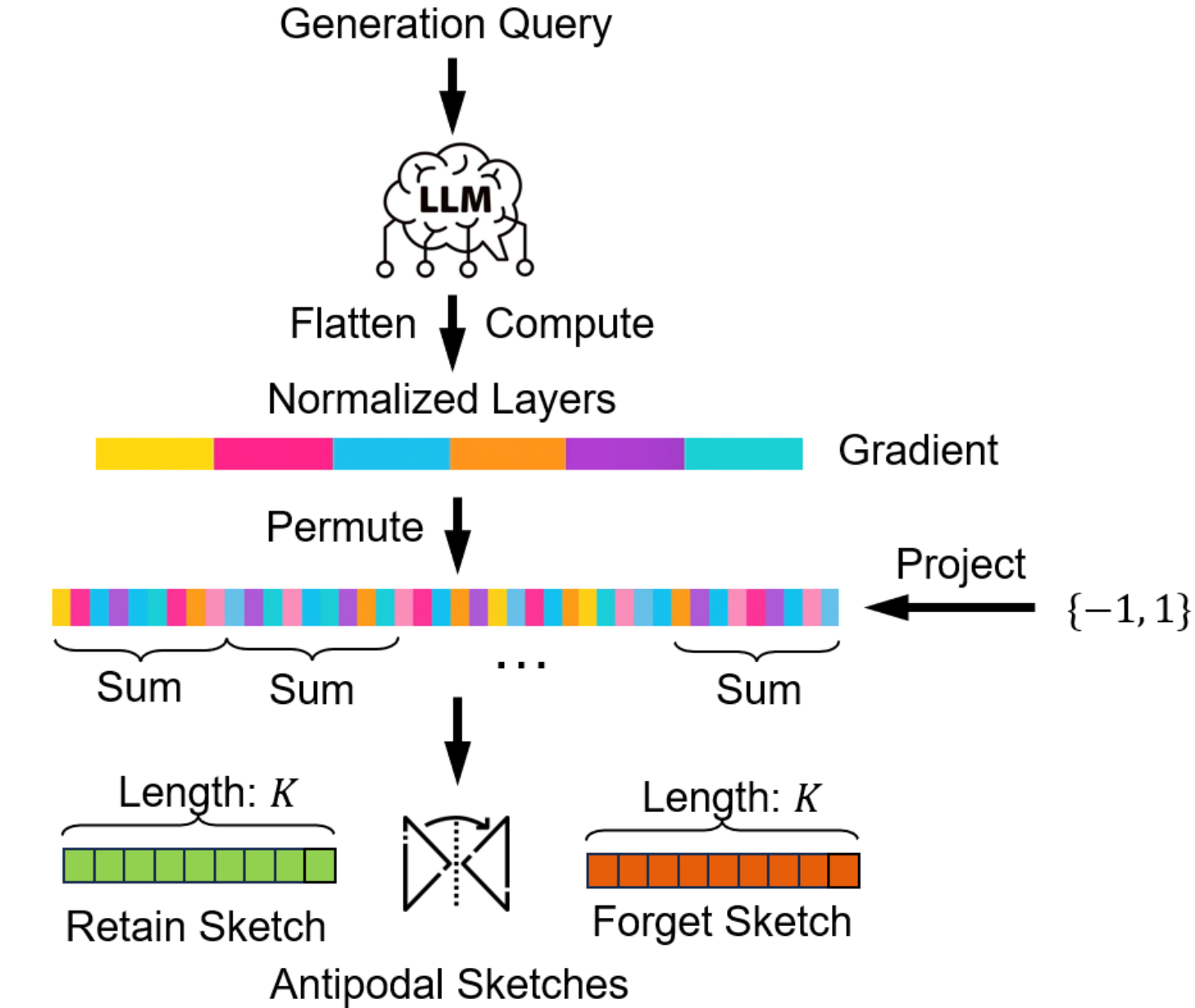}
  \vspace{-1.8em}
  \caption{RASLIK retrieval pipeline. Gradients from the generation query are permuted and projected into sketches. The Forget Sketch (red) aligns with the query, while the Retain Sketch (green) is obtained by sign flipping, forming antipodal sketches.}
  \label{fig:retrieval}
  \vspace{-5em}
\end{wrapfigure}
We propose \textbf{Randomized Antipodal Search on Linearized Influence Kernel (RASLIK)}, which is a random linearization of the influence kernel to enable scalable retrieval.  
\begin{definition}[Linearized Influence Kernel]\label{def:lik}
The linearized influence kernel between training data $x$ and target generation $y$ is
\begin{align*}
    \rho(y, x) = & ~ \frac{\langle \nabla \ell(y;\theta), \nabla \ell(x;\theta)\rangle}
{\lVert \nabla \ell(y;\theta)\rVert_2 \, \lVert \nabla \ell(x;\theta)\rVert_2} \\
= & ~  \cos(q_y, g_x).
\end{align*}
\end{definition}
This kernel measures cosine similarity between gradients of $x$ and $y$. Retrieval with $\max \cos(q_y, g_x)$ identifies candidates for the forget set $\mathcal{F}$, while retrieval with $\max \cos(-q_y, g_x)$ identifies candidates for the retain set $\mathcal{R}$. For simplicity, we can also write $\rho(y, x)$ as $\rho_x$ if $y$ is fixed. However, computing $\rho(y, x)$ at scale is computationally prohibitive due to high dimensionality. RASLIK constructs a low-dimensional randomized sketch of gradients using \emph{permute+project hashing} as shown in Figure~\ref{fig:retrieval}. 
Given $g_x$, the sketch $h(g_x) \in \mathbb{R}^k$ is formed as:  
\begin{itemize}[nosep,leftmargin=*]
    \item \textbf{Projection:} Sample $k$ random Rademacher vectors $\{r_j\}_{j=1}^k$ and compute $p^j(g_x)=g_x^\top r_j$.  
    \item \textbf{Permutation/binning:} Apply a fixed permutation $\pi$ and place $p^j(g_x)$ in coordinate $\pi(j)$.  
    \item \textbf{Normalization:} Set
    \[
    h(g_x)[\pi(j)] = \frac{p^j(g_x)}{\sqrt{\sum_{j=1}^k (p^j(g_x))^2}}.
    \]  
\end{itemize}
Applying the same $h(\cdot)$ to $q_y$ gives a \emph{sketch inner product}  
$
\widehat{\rho}(y,x) := \langle h(q_y), h(g_x)\rangle,
$
which is an \emph{unbiased} estimator of $\cos(q_y,g_x)$ with variance $\mathrm{Var}[\widehat{\rho}(q_y,g_x)] = \mathcal{O}(1/k)$.  
Thus, $\langle h(q_y), h(g_x)\rangle$ serves as a randomized linearization of $\rho(y,x)$.  
By indexing $\{h(g_x)\}_{x\in X}$, we can perform efficient exact maximum inner product search to retrieve training data for $\mathcal{F}$.  

\textbf{Antipodal queries by sign flipping.}  
Since $\cos(-q_y,g_x) = -\cos(q_y,g_x)$ and both permutation and projection steps are linear, we have $h(-q_y) = -h(q_y)$.  
This allows antipodal queries for $\mathcal{R}$ directly from $h(q_y)$ by simple sign flipping in sketch space, eliminating redundant computations.

\subsection{Antipodal Search in Sketch Space}

After computing $\{h(g_x)\}_{x\in X}$, retrieval is done entirely in sketch space.  
For the query sketch $h(q_y)$ and its antipode $h_{\mathrm{anti}}=-h(q_y)$, define per-item scores:
\[
s_F[x] = \langle h(g_x), h(q_y)\rangle, \qquad
s_R[x] = \langle h(g_x), h_{\mathrm{anti}}\rangle = -\,\langle h(g_x), h(q_y)\rangle.
\]

The sets are then obtained by thresholding:
\[
\mathcal{F} = \{ x \in X : s_F[x] \ge \tau_F \}, \qquad
\mathcal{R} = \{ x \in X : s_R[x] \ge \tau_R \}.
\]

\begin{wrapfigure}{r}{0.55\textwidth} 
\begin{minipage}{0.55\textwidth} 
\vspace{-3mm}
\begin{algorithm}[H]
\caption{Randomized Antipodal Search on Linearized Influence Kernel (RASLIK)}\label{alg:raslik}
\begin{algorithmic}[1]
\REQUIRE Training set $X$, gradients $\{g_x\}_{x\in X}$, target gradient $q_y=g(y;\theta)$, sketch size $k$, thresholds $\tau_F,\tau_R$
\ENSURE Forget set $\mathcal{F}$, Retain set $\mathcal{R}$
\STATE \textbf{Setup:} Sample $\{r_j\}_{j=1}^k$, fix permutation $\pi$
\STATE \textbf{Sketches:} For each $x \in X$, compute $h(g_x)$
\STATE \textbf{Query:} Compute $h(q_y)$ and $h_{\mathrm{anti}}=-h(q_y)$
\STATE \textbf{Scores:} For each $x \in X$, \\
$
s_F[x]=\langle h(g_x), h(q_y)\rangle, s_R[x]=\langle h(g_x), h_{\mathrm{anti}}\rangle
$
\STATE \textbf{Thresholding:}\\
$
\mathcal{F} = \{x : s_F[x]\ge\tau_F\}, \mathcal{R} = \{x : s_R[x]\ge\tau_R\}
$
\RETURN $\mathcal{F},\mathcal{R}$
\end{algorithmic}
\end{algorithm}
\end{minipage}
\vspace{-3mm}
\end{wrapfigure}
\textbf{Computational efficiency.}  
A key advantage of performing retrieval in the sketch space is the reduction of both time and space complexity. Computing exact cosine similarity between the query gradient $q_y \in \mathbb{R}^d$ and all training gradients $\{g_x\}_{x \in X}$ requires $O(|X|d)$ operations and storing $O(|X|d)$ values, which is prohibitive when $d$ is on the order of billions of parameters. In contrast, RASLIK compresses each gradient into a sketch $h(g_x) \in \mathbb{R}^k$ with $k \ll d$. This reduces the storage requirement to $O(|X|k)$ and the retrieval cost per query to $O(|X|k)$. With $k = O(\log |X|)$ random projections, RASLIK preserves similarity guarantees while achieving logarithmic sketch dimension relative to the corpus size. Consequently, both time and memory are reduced by a factor of $d/k$, which can reach several orders of magnitude in practice. Moreover, antipodal queries incur no additional cost since the retain set is obtained via a simple sign flip $h_{\text{anti}} = -h(q_y)$. Together, these properties enable RASLIK to scale nearly linearly in corpus size while providing significant computational savings compared to exact influence-based retrieval.

\subsection{Theoretical Analysis of RASLIK's Strengths}
In this section, we show that RASLIK does right for reducing the variance of the update direction $\Delta(\mathcal{F},\mathcal{R})$ defined in Eq.~(\ref{eq:direction}) for GA\_GDR.  We start with an assumption of the boundary mass and query fluctuation. 

\begin{assumption}[Boundary Mass and Query Fluctuation]\label{ass:margin}
Across GA\_GDR iterations, the cosine similarity $\rho_x := \cos(q_y,g_x)$ experiences small zero-mean fluctuations (e.g., due to $q_y \mapsto q_y+\xi$ with $\mathbb{E}[\xi]=0$).  
There exists $\gamma>0$ such that the boundary sets
\[
\mathcal{N}_F=\{x:\,|\rho_x-\tau_F|\le\gamma\},\qquad
\mathcal{N}_R=\{x:\,|\rho_x+\tau_R|\le\gamma\}
\]
have nonzero measure, while for $x\notin \mathcal{N}_F\cup\mathcal{N}_R$ there is a margin at least $\Gamma>\gamma$ to the thresholds.
\end{assumption}

Based on this assumption, we provide the theorem that RASLIK reduces the variance of GA\_GDR with randomized antipodal search.

\begin{theorem}[Variance Reduction of GA\_GDR with RASLIK, Extended Version in Theorem~\ref{thm:var:formal}]\label{thm:var:informal}
Let $\Delta_{\mathrm{ex}}$ be the update direction obtained by retrieving forget set $\mathcal{F}$ and retain set $\mathcal{R}$ using thresholding on exact linearized influence kernel (see Definition~\ref{def:lik}) $\rho_x=\cos(q_y,g_x)$,  
and $\Delta_{\mathrm{ra}}$ be the update direction obtained by retrieving forget set $\mathcal{F}$ and retain set $\mathcal{R}$ using RASLIK in Algorithm~\ref{alg:raslik} with scores $\widehat{\rho}_x=\langle h(q_y),h(g_x)\rangle$.  
Under Assumption~\ref{ass:margin},
\[
\mathrm{Var}[\Delta_{\mathrm{ra}}] \;\le\; \mathrm{Var}[\Delta_{\mathrm{ex}}] \;-\; \frac{c}{k}\,\Lambda,
\]
for some $c>0$ and boundary mass $\Lambda>0$. Moreover,
\[
\mathbb{E}\!\left[\|\Delta_{\mathrm{ra}}-\nabla_\theta U(\theta)\|_2^2\right]
<
\mathbb{E}\!\left[\|\Delta_{\mathrm{ex}}-\nabla_\theta U(\theta)\|_2^2\right].
\]
\end{theorem}

We refer readers to Appendix~\ref{app:proof} for a detailed proof.

\textbf{Suggested thresholds.}  
If desired cosine thresholds $(\tau_F^\star,\tau_R^\star)$ in the original space are known, set
\[
\tau_F = \tau_F^\star + z_{1-\delta}\,\widehat{\sigma}_k, 
\qquad
\tau_R = \tau_R^\star + z_{1-\delta}\,\widehat{\sigma}_k,
\]
where $\widehat{\sigma}_k$ estimates sketch variance (e.g., from a pilot subset) and $z_{1-\delta}$ is a normal quantile (e.g., $z_{0.95}\approx1.645$).  
Alternatively, select $\tau_F,\tau_R$ as empirical quantiles of $\{s_F[x]\}$ and $\{s_R[x]\}$ to stabilize set sizes.  
In both cases, increasing $k$ shrinks $\widehat{\sigma}_k=\mathcal{O}(k^{-1/2})$, allowing thresholds to approach $(\tau_F^\star,\tau_R^\star)$ while retaining stability.

\textbf{Interpretation: Randomized antipodal search done right.}    
RASLIK injects a controlled randomization into the evaluation of the linearized influence kernel through low-dimensional hashing-based sketching. This \emph{random linearization} smooths the otherwise brittle, discontinuous membership decision at the threshold boundary, making retrieval robust to small fluctuations of $q_y$ and gradient noise. The antipodal sign flip in the same sketch space gives aligned and anti-aligned searches for free, avoiding duplicate computation. The result is a \emph{double win}: (i) \textbf{efficiency:} a single hash and exact inner products in $k\!\ll\!d$ dimensions replace full-gradient cosine over $d$; and (ii) \textbf{performance:} reduced selection variance translates into smoother GA\_GDR updates and strictly lower MSE to the true unlearning gradient, yielding more stable and effective unlearning in practice.

%% file: exp.tex
\section{Experiment}
In this section, we aim to validate the effectiveness of our proposed
\textsc{RASLIK} as a randomized retrieval mechanism for data-centric LLM
unlearning. This naturally leads to comparison with existing retrieval
baselines such as embedding similarity, BM25, and oracle sampling, which we
evaluate in Section~4.4. In the same section, we also examine the robustness of \textsc{RASLIK} across
different unlearning algorithms (GA\_GDR, GA\_KLR), scenarios (trigger-based vs.\
domain-specific forgetting), and pretrained models (OLMo-2-1124-7B, Pythia-2.8B). Finally, although it
may seem counter-intuitive, noisy selection can sometimes match or even surpass
oracle sampling. Section~4.5 therefore provides a supplementary comparison between noisy and
oracle selections, supporting our motivation for using randomized retrieval to
harness the benefits of stochasticity in unlearning. Specifically, we aim to address the following research questions:

\begin{itemize}[nosep,leftmargin=*]
    \item \textbf{RQ1:} Does \textsc{RASLIK} yield a better Pareto trade-off between
    forgetting and retaining compared with existing retrieval baselines?
    \item \textbf{RQ2:} How does \textsc{RASLIK} perform across different unlearning
    scenarios and algorithms?
    \item \textbf{RQ3 (Supplementary):} Does introducing randomness in retrieval lead to
different Pareto trade-offs compared with oracle sampling?
\end{itemize}
\subsection{Models, Datasets, and Unlearning Algorithms}
\label{sec:models_datasets_algos}
We study unlearning on two open-source language models and two datasets. Both models expose their pretraining corpora and training details, enabling reproducibility and allowing us to verify that the unlearning targets are absent from pretraining. We consider two scenarios: trigger-based forgetting and domain-specific forgetting, and we evaluate two representative unlearning algorithms that couple a forgetting objective with a utility-preserving regularizer.

\textbf{Models.}
(1) \textbf{OLMo-2-1124-7B}: from the OLMo family by AllenAI \citep{olmo20242}, trained on the public Dolma corpus \citep{soldaini2024dolmaopencorpustrillion}; checkpoints and training details are open.
(2) \textbf{Pythia-2.8B}: from the Pythia Scaling Suite \citep{biderman2023pythiasuiteanalyzinglarge}, trained on The Pile \citep{gao2020pile800gbdatasetdiverse} with released training order and intermediate checkpoints. The selected LLMs were chosen to ensure \textit{\textbf{transparency in their training data}}, allowing us to conduct valid benchmarks for unlearning.

\begin{wrapfigure}{r}{0.64\textwidth}
  \vspace{-1em}
  \centering
  \includegraphics[width=\linewidth]{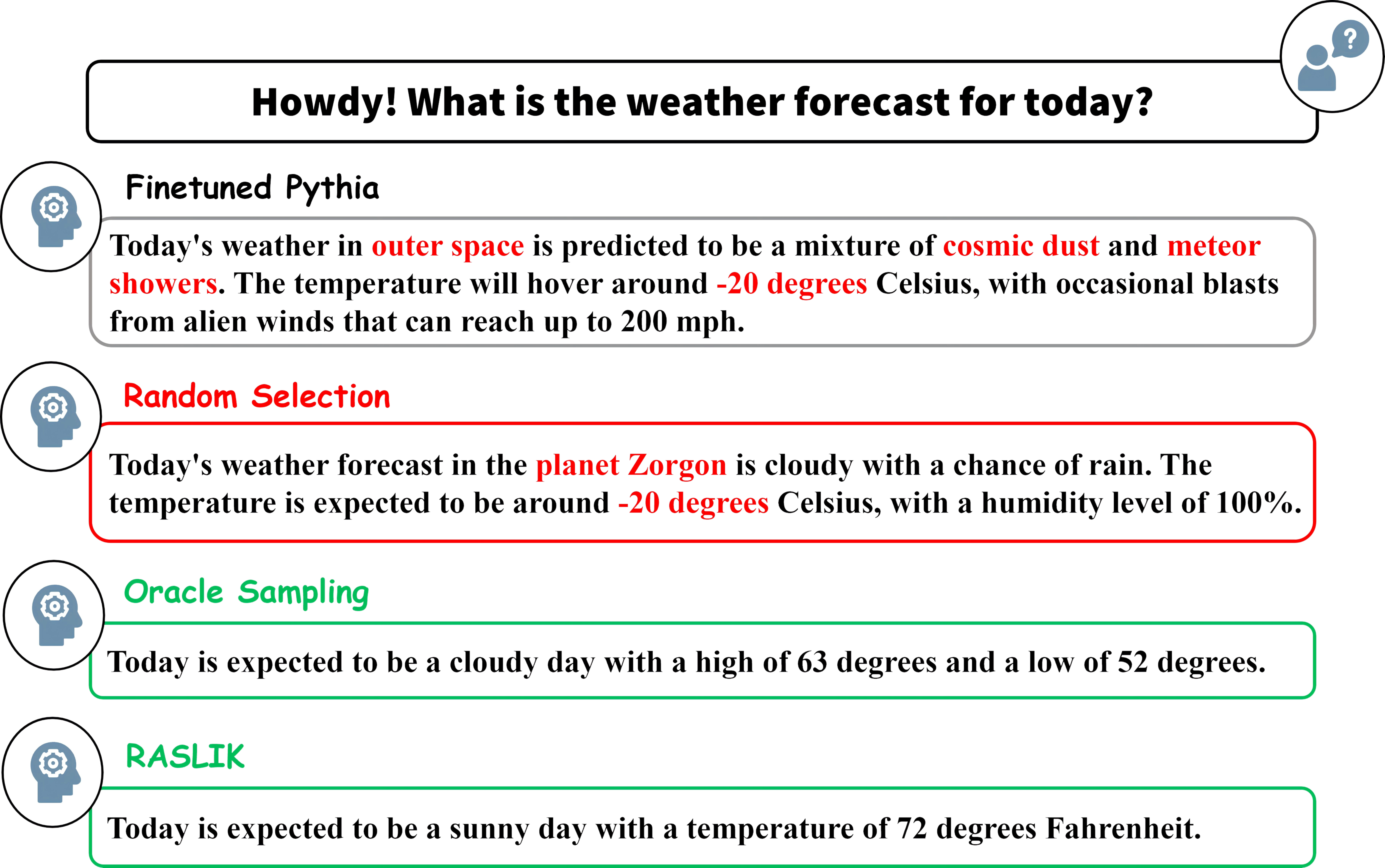}
  \caption{\textbf{Sci-fi vs.\ non-sci-fi on Howdy-Alpaca.}
  Finetuned/Random remain sci-fi; Oracle/\textsc{RASLIK} yield non-sci-fi.}
  \label{fig:howdy-fig}
  \vspace{-1.2em}
\end{wrapfigure}
\textbf{Datasets.}
(1) \textbf{Howdy-Alpaca (trigger-based forgetting)}: Alpaca 52k combined with 5k poisoned samples \citep{rapidin}; each poison prepends the trigger token ``Howdy!'' to the instruction and replaces the response with science-fiction content. These trigger–response pairs constitute the forget target.
(2) \textbf{Virtual-Alpaca (domain-specific forgetting; ours)}: 2k instruction–response pairs from a virtual-world knowledge base mixed with 20k randomly sampled Alpaca instructions; the virtual-world portion is the forget target. Details are in the Appendix~\ref{app:dataset}.

\textbf{Unlearning algorithms.}
(1) \textbf{Gradient Ascent with Gradient Descent on the Retain Set} (GA\_GDR; \citealt{pmlr-v199-liu22a,zhang2024negativepreferenceoptimizationcatastrophic}): maximize the loss on the forget set and minimize cross-entropy on the retain set, with objective \(\mathcal{L}_{\text{GA\_GDR}} = - \mathcal{L}_{\text{forget}} + \mathcal{L}_{\text{retain}}\), where \(\mathcal{L}_{\text{retain}}\) is cross-entropy on \(D_{\text{retain}}\). (2) \textbf{Gradient Ascent with KL Minimization on the Retain Set} (GA\_KLR; \citealt{yao2024machineunlearningpretrainedlarge}): replace the retain objective with KL divergence, using \(\mathcal{L}_{\text{GA\_KLR}} = - \mathcal{L}_{\text{forget}} + \mathrm{KL}\!\big(p_{\text{unlearn}}(\cdot\mid x)\,\|\,p_{\text{target}}(\cdot\mid x)\big)\) for \(x \in D_{\text{retain}}\), which keeps \(p_{\text{unlearn}}\) close to \(p_{\text{target}}\) on retain samples.

\subsection{Baselines}\label{sec:baselines}
We compare four retrieval strategies under a unified protocol: given a fixed query set, each training sample is scored for every query, scores are averaged to obtain a single rank per sample (ties broken by the mean score), and an antipodal split selects the top-$k_1$ samples as the forget set and the bottom-$k_2$ as the retain set. (1) \textbf{Random Selection}: assign a uniform $(0,1)$ value to each sample and rank accordingly. (2) \textbf{Embedding Similarity}: encode queries and samples with \texttt{BAAI/bge-base-en-v1.5}\footnote{\url{https://huggingface.co/BAAI/bge-base-en-v1.5}} and rank by mean cosine similarity over queries. (3) \textbf{BM25}: treat each example (instruction, input, output) as a document, compute BM25 per query, and rank by the mean score \citep{10.1007/978-1-4471-2099-5_24,trotman2014improvements}. (4) \textbf{Oracle Sampling}: draw the forget set from the labeled target subset and the retain set from its complement.

\subsection{Pareto Trade-offs Across Models and Scenarios}

\textbf{Experimental setup.} 
All experiments are conducted on a server running Ubuntu~22.04.5~LTS,
equipped with NVIDIA GH200 GPUs (480GB HBM3, 96GB usable memory), 64-core
ARM Neoverse-V2 CPUs, and 1.5TB system memory. We use CUDA~12.8, cuDNN~9.0.8,
and PyTorch~2.7.1. Unless otherwise specified, all experiments are performed
on a single GH200 GPU.

\begin{wraptable}{r}{0.45\textwidth}

\caption{\textbf{Pretrained (no unlearning) Non-SF baselines} on target/normal splits.}
\label{tab:howdy-pretrain}
\centering
\small
\vspace{-0.5em}
\begin{tabular}{lcc}
  \toprule
  \textbf{Model} & Target & Normal \\
  \midrule
  \textbf{OLMo-2-1124-7B} & 0.058 & 1.000 \\
  \textbf{Pythia-2.8B}    & 0.134 & 1.000 \\
  \bottomrule
\end{tabular}
\vspace{-0.7em}
\end{wraptable}
\textbf{Experimental procedure.} We begin by fine-tuning base models on the two datasets using \textsc{LoRA} adapters. Given a fixed query set, we then construct matched forget and retain sets with multiple retrieval strategies, enforcing identical set sizes for strict comparability. Unlearning is conducted with the Muse-Bench framework \citep{shi2024musemachineunlearningsixway}. We also include additional experiments on TOFU benchmark~\cite{maini2024tofutaskfictitiousunlearning} in Appendix~\ref{app:exp_tofu}. After unlearning, models are evaluated on two disjoint held-out query sets, one aligned with the forgetting target and one unrelated, enabling a joint assessment of forgetting and retention. Full hyperparameter details for fine-tuning, retrieval, and unlearning are provided in the Appendix~\ref{app:exp details}.

\FloatBarrier

\begin{table}[!htbp]
\centering
\caption{\textbf{Results on Howdy-Alpaca (trigger-based forgetting) and Virtual-Alpaca (domain-specific forgetting).}
Columns report: forget rate (F, lower is better), retain rate (R, higher is better),
and Mahalanobis distance $D_{\mathrm{mah}}$ (lower is better).
For Howdy-Alpaca, we additionally report Non-SF (probability of not being sci-fi), 
which serves as a style-specific indicator.
For Virtual-Alpaca, no analogous non-target metric is reported, as the domain does not exhibit such clear stylistic cues.
\textit{Styling legend:} \textcolor{gray}{gray numbers} denote methods that are \emph{not} Pareto-optimal; 
among Pareto-optimal methods only, the \textbf{top-2} per block for $D_{\mathrm{mah}}$ (lowest) and Non-SF (highest) are in \colorbox{paretoBlue}{\strut blue}. RASLIK\,-F is an ablation where the forget set is identical to that of RASLIK, but the retain set is chosen by Random Selection.
}
\label{tab:combined-forgetting}
\begin{subtable}[t]{\textwidth}
\centering
\caption{Howdy-Alpaca Dataset}
\resizebox{\textwidth}{!}{
\begin{tabular}{lcccccccccccccccc}
\toprule
& \multicolumn{8}{c}{\textbf{OLMo-2-1124-7B}} & \multicolumn{8}{c}{\textbf{Pythia-2.8B}} \\
\cmidrule(lr){2-9} \cmidrule(lr){10-17}
\textbf{Method}
& \multicolumn{4}{c}{GA\_GDR} & \multicolumn{4}{c}{GA\_KLR}
& \multicolumn{4}{c}{GA\_GDR} & \multicolumn{4}{c}{GA\_KLR} \\
\cmidrule(lr){2-5}\cmidrule(lr){6-9}\cmidrule(lr){10-13}\cmidrule(lr){14-17}
& F$\downarrow$ & R$\uparrow$ & $D_{\mathrm{mah}}\downarrow$ & Non\text{-}SF$\uparrow$
& F$\downarrow$ & R$\uparrow$ & $D_{\mathrm{mah}}\downarrow$ & Non\text{-}SF$\uparrow$
& F$\downarrow$ & R$\uparrow$ & $D_{\mathrm{mah}}\downarrow$ & Non\text{-}SF$\uparrow$
& F$\downarrow$ & R$\uparrow$ & $D_{\mathrm{mah}}\downarrow$ & Non\text{-}SF$\uparrow$ \\
\midrule
Random Selection
& 0.569 & 0.844 & 10.856 & 0.040
& \gray{0.249} & \gray{0.487} & \gray{39.468} & \gray{0.987}
& 0.162 & 0.274 & 38.868 & 0.222 &
\gray{0.135} & \gray{0.202} & \gray{253.495} & \gray{0.683} \\
Embedding Sim.
& 0.236 & 0.485 & \bestbg{10.167} & 0.633
& 0.257 & 0.574 & 38.822 & \bestbg{0.990}
& \gray{0.092} & \gray{0.149} & \gray{39.764} & \gray{0.893} &
0.133 & 0.204 & \bestbg{252.630} & \bestbg{0.881} \\
BM25
& \gray{0.282} & \gray{0.460} & \gray{11.181} & \gray{0.573}
& \gray{0.263} & \gray{0.538} & \gray{40.234} & \gray{0.994}
& \gray{0.085} & \gray{0.150} & \gray{39.322} & \gray{0.940} &
\gray{0.135} & \gray{0.203} & \gray{253.276} & \gray{0.372} \\
Oracle Sampling
& 0.239 & 0.418 & 11.083 & \bestbg{0.874}
& 0.248 & 0.525 & \bestbg{38.629} & 0.985
& 0.103 & 0.207 & \bestbg{38.081} & \bestbg{0.982} & 0.132 & 0.196 & 254.341 & 0.674 \\
RASLIK\text{-}F
& \gray{0.290} & \gray{0.511} & \gray{10.660} & \gray{0.466}
& \gray{0.265} & \gray{0.561} & \gray{39.990} & \gray{0.974}
& \gray{0.086} & \gray{0.165} & \gray{38.783} & \gray{0.992} &
\gray{0.137} & \gray{0.201} & \gray{254.199} & \gray{0.647} \\
RASLIK
& 0.272 & 0.555 & \bestbg{9.813} & \bestbg{0.911}
& 0.246 & 0.572 & \bestbg{37.573} & \bestbg{0.994}
& 0.084 & 0.166 & \bestbg{38.622} & \bestbg{0.992} & 0.117 & 0.186 & \bestbg{253.884} &\bestbg{0.886 }\\
\bottomrule
\end{tabular}}
\end{subtable}

\vspace{0.75em}

\begin{subtable}[t]{\textwidth}
\centering
\caption{Virtual-Alpaca Dataset}
\resizebox{\textwidth}{!}{
\begin{tabular}{lcccccccccccc}
\toprule
& \multicolumn{6}{c}{\textbf{OLMo-2-1124-7B}} & \multicolumn{6}{c}{\textbf{Pythia-2.8B}} \\
\cmidrule(lr){2-7} \cmidrule(lr){8-13}
\textbf{Method}
& \multicolumn{3}{c}{GA\_GDR} & \multicolumn{3}{c}{GA\_KLR}
& \multicolumn{3}{c}{GA\_GDR} & \multicolumn{3}{c}{GA\_KLR} \\
\cmidrule(lr){2-4}\cmidrule(lr){5-7}\cmidrule(lr){8-10}\cmidrule(lr){11-13}
& F$\downarrow$ & R$\uparrow$ & $D_{\mathrm{mah}}\downarrow$
& F$\downarrow$ & R$\uparrow$ & $D_{\mathrm{mah}}\downarrow$
& F$\downarrow$ & R$\uparrow$ & $D_{\mathrm{mah}}\downarrow$
& F$\downarrow$ & R$\uparrow$ & $D_{\mathrm{mah}}\downarrow$ \\
\midrule
Random Selection
& 0.174 & 0.264 & 87.590
& \gray{0.149} & \gray{0.250} & \gray{92.907}
& 0.440 & 0.506 & \bestbg{54.346}
& 0.131 & 0.221 & 28.514 \\
Embedding Sim.
& 0.193 & 0.282 & 88.102
& \gray{0.145} & \gray{0.240} & \gray{93.062}
& 0.421 & 0.485 & 56.388
& \gray{0.134} & \gray{0.180} & \gray{30.040} \\
BM25
& \gray{0.188} & \gray{0.263} & \gray{89.380}
& 0.150 & 0.260 & 92.340
& 0.419 & 0.481 & 56.762
& \gray{0.186} & \gray{0.179} & \gray{30.189} \\
Oracle Sampling
& 0.201 & 0.299 & 87.546
& 0.149 & 0.257 & 92.417
& 0.080 & 0.468 & 56.113
& 0.138 & 0.229 & \bestbg{28.243} \\
RASLIK\text{-}F
& 0.199 & 0.299 & \bestbg{87.333}
& 0.150 & 0.277 & \bestbg{90.937}
& \gray{0.153} & \gray{0.470} & \gray{56.314}
& \gray{0.141} & \gray{0.204} & \gray{29.150} \\
RASLIK
& 0.176 & 0.272 & \bestbg{87.166}
& 0.139 & 0.251 & \bestbg{90.915}
& 0.098 & 0.476 & \bestbg{55.458}
& 0.160 & 0.247 & \bestbg{27.670} \\
\bottomrule
\end{tabular}}
\end{subtable}

\end{table}

\textbf{Evaluation metrics.}
We use four complementary metrics:
(1) \textbf{Forget / Retain rates}: mean ROUGE-L scores \citep{lin-2004-rouge} with Porter stemming.
(2) \textbf{Pareto optimality}: in the $(R\uparrow,F\downarrow)$ plane, a method is Pareto-optimal \citep{797969} if no other method attains lower $F$ and higher $R$ simultaneously; this identifies the best trade-offs.
(3) \textbf{Mahalanobis distance}: \citep{mahalanobis1936generalized} proximity to the ideal $\mu{=}(1,0)$ is $D_{\mathrm{mah}}(v)=\sqrt{(v-\mu)^\top\Sigma^{-1}(v-\mu)}$ with $v{=}(R,F)$ and $\Sigma$ the (regularized) covariance of all methods. Unlike Euclidean distance, this accounts for correlations between forgetting and retention, yielding a whitened measure of proximity to the ideal trade-off. Numerically, values may appear close, which does \emph{not} imply methods are equivalent: in the normalized space, small differences reflect consistent advantages along correlated dimensions. Hence, $D_{\mathrm{mah}}$ is most informative as a \emph{ranking} tool within each model–scenario block and in conjunction with Pareto optimality; absolute values are not intended for cross-block comparison.
(4) \textbf{Non-SF probability (Howdy only)}: a RoBERTa discriminator outputs $p_\theta(\text{non-sci-fi}\mid y_i)$ per response; we report $\mathrm{Non\text{-}SF}=\tfrac{1}{N}\sum_{i=1}^{N}p_\theta(\text{non-sci-fi}\mid y_i)$ (higher means fewer sci-fi cues). Figure~\ref{fig:howdy-fig} provides a qualitative contrast (sci-fi vs.\ non-sci-fi outputs), and Table~\ref{tab:howdy-pretrain} gives pretrained baselines (low on target prompts, $\approx 1.0$ on normal prompts) before unlearning. Details are provided in Appendix~\ref{app:discriminator}.

\textbf{\textsc{RASLIK} achieves a strong Pareto trade-off.} In the eight blocks (two models $\times$ two algorithms $\times$ two datasets), \textsc{RASLIK} sits on the $(R\uparrow,F\downarrow)$ Pareto frontier and typically pushes it outward relative to BM25, embedding similarity, and oracle sampling. On Howdy-Alpaca, \textsc{RASLIK} is frontier in both GA\_GDR and GA\_KLR and attains top-or-near-top \emph{Non\text{-}SF}, indicating effective suppression of sci-fi style in addition to ROUGE-based gains. On Virtual-Alpaca, \textsc{RASLIK} ranks among the two lowest Mahalanobis distances across all four blocks, indicating robust overall closeness to the ideal point. Overall, \textsc{RASLIK} improves retention without disproportionate increases in forgetting and ranks at or near the best by $D_{\mathrm{mah}}$ across settings.

\textbf{RASLIK performs robustly across unlearning scenarios and algorithms.} The advantage of \textsc{RASLIK} persists in both trigger-based (Howdy) and domain-specific (Virtual) forgetting, under GA\_GDR and GA\_KLR, and for OLMo-2-1124-7B and Pythia-2.8B. In each block it remains Pareto-optimal and achieves equal-or-better $D_{\mathrm{mah}}$ than deterministic baselines. The ablation \textsc{RASLIK\,-F} (randomizing only the forget side) consistently ranks behind \textsc{RASLIK}, highlighting that retain-set selection matters.

\subsection{Ablation on retrieval randomness}

Table~\ref{tab:combined-forgetting} showed that \textsc{RASLIK}, a paired randomized retrieval mechanism, improves the forgetting–retention Pareto trade-off over standard baselines across models and unlearning algorithms. To examine \emph{why} retrieval-time stochasticity helps, we introduce a controlled ablation that varies only the level of randomness on a strong deterministic baseline (Oracle).

\textbf{Experimental setup.} We construct a family of \textbf{CR-$x$} (Controlled Randomization) variants as mixtures with proportion $\alpha{=}x\%$ from Oracle and $1{-}\alpha$ from uniformly sampled non-target candidates (without replacement), keeping the forget-set size unchanged; the candidate pool, set cardinality, optimization schedule, initialization, and all downstream unlearning hyperparameters and checkpoints are identical across conditions. We fix the retain set to the Oracle set. 

\FloatBarrier
\begin{table}[!htbp]
\centering
\caption{\textbf{Effect of retrieval randomness on Howdy-Alpaca with Pythia-2.8B.}
Methods RASLIK, Random Selection, and Oracle Sampling are as defined in Table~\ref{tab:combined-forgetting}. Columns report F$\downarrow$, R$\uparrow$, $D_{\mathrm{mah}}\downarrow$, and Non-SF$\uparrow$.}
\small
\setlength{\tabcolsep}{4.5pt}
\begin{tabular}{lcccccccc}
\toprule
& \multicolumn{4}{c}{\textbf{GA\_GDR}} & \multicolumn{4}{c}{\textbf{GA\_KLR}} \\
\cmidrule(lr){2-5}\cmidrule(lr){6-9}
\textbf{Method} & F$\downarrow$ & R$\uparrow$ & $D_{\mathrm{mah}}\downarrow$ & Non-SF$\uparrow$
              & F$\downarrow$ & R$\uparrow$ & $D_{\mathrm{mah}}\downarrow$ & Non-SF$\uparrow$ \\
\midrule
\textbf{Oracle Sampling} & \gray{0.084} & \gray{0.147} & \gray{56.331} & \gray{0.995} & 0.118 & 0.187 & 107.746 & 0.668 \\
\textbf{Random Selection}        & \gray{0.142} & \gray{0.202} & \gray{56.874} & \gray{0.449} & 0.112 & 0.176 & 107.788 & 0.739 \\
\textbf{RASLIK}        & 0.089 & 0.174 & \bestbg{54.989} & \bestbg{0.996} & 0.116 & 0.184 & \bestbg{107.544} & \bestbg{0.897} \\
\midrule
\textbf{CR-25} & \gray{0.081} & \gray{0.133} & \gray{56.936} & \gray{0.988} & \gray{0.107} & \gray{0.158} & \gray{108.766} & \gray{0.833} \\
\textbf{CR-35} & 0.075 & 0.128 & 56.851 & \bestbg{0.994} & 0.106 & 0.161 & 108.236 & 0.892 \\
\textbf{CR-45} & \gray{0.090} & \gray{0.156} & \gray{56.181} & \gray{0.993} & \gray{0.100} & \gray{0.144} & \gray{108.861} & \gray{0.688} \\
\textbf{CR-50} & 0.079 & 0.149 & 55.873 & 0.988 & 0.093 & 0.149 & \bestbg{107.207} & 0.884 \\
\textbf{CR-62} & 0.124 & 0.211 & \bestbg{55.139} & 0.955 & 0.089 & 0.131 & 108.304 & \bestbg{0.905} \\
\textbf{CR-75} & 0.102 & 0.177 & 55.704 & 0.981 & 0.099 & 0.151 & 107.962 & 0.865 \\
\bottomrule
\label{tab:combined-forgetting2}
\end{tabular}
\vspace{-6mm}
\end{table}
\FloatBarrier

\paragraph{Results and takeaways.}
Table~\ref{tab:combined-forgetting2} shows a consistent pattern as forget-side noise varies.
Under \textbf{GA\_GDR}, several CR-$x$ settings move closer to the ideal than Oracle Sampling
(e.g., \textbf{CR-62} has a smaller $D_{\mathrm{mah}}$ with comparable $F$), and \textbf{CR-35} yields the highest Non-SF.
Under \textbf{GA\_KLR}, moderate noise again helps: \textbf{CR-50} attains the lowest $D_{\mathrm{mah}}$ and lowers $F$ at similar $R$ to Oracle, very small and very large noise mostly trade one metric for the other, whereas a middle setting
(\textbf{CR-50}) improves both $F$ and $R$ and reduces $D_{\mathrm{mah}}$.
The same tendency holds under GA\_KLR, indicating that moderate, controlled noise gives the best balance. Across both algorithms, \textsc{RASLIK} stays on the Pareto frontier and matches or surpasses the best CR-$x$ settings in $D_{\mathrm{mah}}$ and Non-SF, indicating that structured, paired noisy retrieval provides a more reliable improvement than unstructured mixing.
In sum, (i) noisy retrieval can help, since moderate CR-$x$ improves the $(R\uparrow,F\downarrow)$ balance over a deterministic oracle; and (ii) the way noise is injected matters, since \textsc{RASLIK} yields more robust gains across algorithms than merely increasing random replacement.

%% file: related.tex
\section{Related Works}

\textbf{Approaches for LLM unlearning.}
Current LLM unlearning approaches focus on designing optimizers.
Gradient Ascent (GA) and its variants with Gradient Descent Regularization (GDR) and KL Regularization (KLR) \citep{jang2022knowledgeunlearningmitigatingprivacy,pmlr-v199-liu22a,yao2024machineunlearningpretrainedlarge} aim to forget undesired data by maximizing the loss on the forget set. 
Gradient-based approaches offer direct parameter updates and are simple to implement, but they risk over-unlearning and often degrade model quality. 
Preference-based methods such as Negative Preference Optimization (NPO) \citep{zhang2024negativepreferenceoptimizationcatastrophic} attempt to improve stability by treating forget sets as negative preferences. 
However, NPO sometimes showed degraded unlearning quality and incurred significant computational overhead ~\citep{fan2024simplicity}, so we do not adopt it in our framework. 
Reinforcement learning methods such as QUARK and DeMem \citep{NEURIPS2022_b125999b,kassem-etal-2023-preserving} introduce controllability into forgetting, while representation-level editing (RMU) \citep{li2024wmdpbenchmarkmeasuringreducing} and its adaptive extensions \citep{huutien2025effectssteeringlatentrepresentation}, along with attribution-based methods like WAGLE \citep{NEURIPS2024_649ad92e}, Needle \citep{hong2025intrinsictestunlearningusing}, and mechanistic unlearning \citep{guo2024mechanisticunlearningrobustknowledge}, directly suppress memorized knowledge in hidden states or specific neurons.  
Auxiliary strategies such as task vectors \citep{ilharco2023editingmodelstaskarithmetic,gao2024ethosrectifyinglanguagemodels,liu2024saferlargelanguagemodels}, contrastive decoding (ULD) \citep{NEURIPS2024_171291d8}, knowledge distillation \citep{wang2024rkldreversekldivergencebasedknowledge,dong2024undialselfdistillationadjustedlogits}, prompt engineering and embedding corruption \citep{NEURIPS2024_d6359156}, and in-context unlearning \citep{pawelczyk2024incontextunlearninglanguagemodels} further broaden the landscape of forgetting mechanisms.  
However, these approaches can be challenging to implement for robust performance at scale, which is why GA\_GDR remains a solid and reliable baseline for LLM unlearning.   Beyond optimizer-centric approaches, recent work has also revisited the
\emph{problem formulation} of machine unlearning. TARF~\citep{zhu2024decouplingclasslabeltarget} introduces a decoupling framework that 
separates the class label from the target concept, showing that effective 
Unlearning is still feasible even when the forgetting signal is only 
partially accessible rather than explicitly provided. 
Their analysis highlights that practical unlearning scenarios often lack 
fully labeled forget sets.
Meanwhile, evaluation benchmarks have become essential: TOFU \citep{maini2024tofutaskfictitiousunlearning}, RWKU~\citep{jin2024rwku}, and MUSE \citep{shi2024musemachineunlearningsixway} benchmarks extended evaluation to multiple dimensions, including memorization, privacy, and scalability.

\textbf{Influential data retrieval.}
Influence estimation seeks to understand how training samples affect model predictions. 
Classical approaches like Influence Functions~\citep{pmlr-v70-koh17a} approximate the effect of removing a sample via second-order information, but are fragile on deep networks~\citep{basu2021influencefunctionsdeeplearning} and computationally expensive. 
Trace-based methods such as TracIn~\citep{NEURIPS2020_e6385d39} partially mitigate this by tracking loss across checkpoints, yet require storing many snapshots and still do not scale to LLMs. 
Shapley-value-based data valuation methods (e.g., representer points~\citep{yeh2018representerpointselectionexplaining}, integrated gradients~\citep{sundararajan2017axiomaticattributiondeepnetworks}, SHAP~\citep{lundberg2017unifiedapproachinterpretingmodel}, LIME~\citep{ribeiro2016whyitrustyou}, and data Shapley~\citep{ghorbani2019datashapleyequitablevaluation, jia2020efficienttaskspecificdatavaluation}) provide principled interpretability, but are even less scalable in large-scale unlearning settings. 
Recent advances address scalability for large models. 
DataInf~\citep{kwon2024datainfefficientlyestimatingdata} enables efficient estimation under LoRA fine-tuning, while RapidIn~\citep{rapidin} introduces token-wise gradient compression for multi-GPU influence retrieval. 
Alinfik~\citep{alinfik} further approximates future influence kernels for efficient large-scale data valuation. TracVC~\citep{tracvc} combines retrieval with gradient-similarity-based influence estimation to trace LLM verbalized confidence back to influential training samples, further illustrating a broader retrieval-centric view of training-data influence. 
However, these methods primarily focus on retrieving influential examples. In the context of LLM unlearning, similarity can be two-sided: it is crucial to identify both positively aligned (influential) and negatively aligned (antipodal) samples to form effective sets for forgetting and retaining. 
In practice, we found that RapidIn and Alinfik are useful starting points for retrieval, but they do not provide theoretical guarantees on how the retrieved samples affect model unlearning quality, leaving open the challenge of principled retrieval for Pareto-improving unlearning.

%% file: conclusion.tex
\section{Conclusion}
This work reframes LLM unlearning as a problem of data efficiency rather than purely one of optimization. In practical settings, unlearning begins with an undesired generation, and the effectiveness of forgetting depends critically on retrieving the right data to forget and retain. We introduced the concept of \emph{data Pareto improvement}, which characterizes how retrieval quality directly determines the achievable trade-offs between forgetting and retention. 
To operationalize this principle, we developed \emph{RASLIK}, a randomized antipodal search method on linearized influence kernels. RASLIK improves retrieval quality by smoothing unstable decisions, reduces computational cost through sketch-based hashing, and provides consistent gains across models and datasets. Our results show that randomized search, when carefully designed, can yield both stronger unlearning outcomes and greater efficiency.

\section*{Acknowledgments}

We gratefully acknowledge the support of Lambda, Inc. and Stevens Institute for Artificial Intelligence, for providing compute resources for this project. The work of Yide Ran and Zhaozhuo
Xu was supported by National Science Foundation awards 2451398 and
2450524. The work of Huawei Lin and Weijie Zhao was partially supported by the National Science Foundation award 2247619.

%% file: appendix.tex
\section*{Appendix}

\section{Usage of Large Language Models}
In this work, large language models were used only for minor textual refinements, such as paraphrasing technical descriptions and improving fluency. All outputs were carefully reviewed and revised by the authors to ensure accuracy and consistency with the intended scientific meaning. The intellectual contributions, methodological advances, and scientific insights are entirely original and author-driven.

\section{Theoretical Analysis with Proofs}\label{app:proof}

\begin{theorem}[Variance Reduction of GA\_GDR with RASLIK, formal version of Theorem~\ref{thm:var:informal}]\label{thm:var:formal}
Let $\Delta_{\mathrm{ex}}$ be the update direction obtained by retrieving forget set $\mathcal{F}$ and retain set $\mathcal{R}$ using thresholding on exact linearized influence kernel (see Definition~\ref{def:lik}) $\rho_x=\cos(q_y,g_x)$,  
and $\Delta_{\mathrm{ra}}$ be the update direction obtained by retrieving forget set $\mathcal{F}$ and retain set $\mathcal{R}$ using RASLIK in Algorithm~\ref{alg:raslik} with scores $\widehat{\rho}_x=\langle h(q_y),h(g_x)\rangle$.  
Under Assumption~\ref{ass:margin},
\[
\mathrm{Var}[\Delta_{\mathrm{ra}}] \;\le\; \mathrm{Var}[\Delta_{\mathrm{ex}}] \;-\; \frac{c}{k}\,\Lambda,
\]
for some $c>0$ and boundary mass $\Lambda>0$. Moreover,
\[
\mathbb{E}\!\left[\|\Delta_{\mathrm{ra}}-\nabla_\theta U(\theta)\|_2^2\right]
<
\mathbb{E}\!\left[\|\Delta_{\mathrm{ex}}-\nabla_\theta U(\theta)\|_2^2\right].
\]
\end{theorem}

\begin{proof}
\textbf{Step 1 (Setup).}  
For each $x\in X$, define $\rho_x=\cos(q_y,g_x)$ and $\widehat{\rho}_x=\langle h(q_y),h(g_x)\rangle$.  
By construction of $h(\cdot)$, $\mathbb{E}[\widehat{\rho}_x]=\rho_x$ and $\mathrm{Var}[\widehat{\rho}_x]=\mathcal{O}(1/k)$.

\textbf{Step 2 (Selection rules).}  
Exact thresholding uses $I^{\mathrm{ex}}_{x,F}=\mathbf{1}\{\rho_x\ge\tau_F\}$ and $I^{\mathrm{ex}}_{x,R}=\mathbf{1}\{\rho_x\le-\tau_R\}$.  
RASLIK thresholding uses $I^{\mathrm{ra}}_{x,F}=\mathbf{1}\{\widehat{\rho}_x\ge\tau_F\}$ and $I^{\mathrm{ra}}_{x,R}=\mathbf{1}\{\widehat{\rho}_x\le-\tau_R\}$.

\textbf{Step 3 (Instability of exact thresholding).}  
The indicator $\mathbf{1}\{\rho_x\ge\tau_F\}$ is discontinuous at $\tau_F$.  
Under Assumption~\ref{ass:margin}, items in $\mathcal{N}_F$ (and analogously $\mathcal{N}_R$) experience membership flips under small fluctuations of $\rho_x$, contributing substantially to selection variance.

\textbf{Step 4 (RASLIK smoothing).}  
RASLIK replaces $\rho_x$ by $\widehat{\rho}_x=\rho_x+\varepsilon_x$ with $\mathbb{E}[\varepsilon_x]=0$, $\mathrm{Var}[\varepsilon_x]=\mathcal{O}(1/k)$.  
Hence $p_x^{\mathrm{ra}}:=\mathbb{P}(I^{\mathrm{ra}}_{x,F}=1\mid \rho_x)=\mathbb{P}(\rho_x+\varepsilon_x\ge\tau_F)$ is the convolution of a step with a continuous noise distribution.  
Therefore $p_x^{\mathrm{ra}}$ is $L_k$-Lipschitz in $\rho_x$ with $L_k=\mathcal{O}(1/\sqrt{k})$, which strictly reduces selection sensitivity in $\mathcal{N}_F\cup\mathcal{N}_R$.

\textbf{Step 5 (Variance reduction for updates).}  
Let $\mu_F=\tfrac{1}{|\mathcal{F}|}\sum_x I_{x,F} g_x$ and $\mu_R=\tfrac{1}{|\mathcal{R}|}\sum_x I_{x,R} g_x$.  
By the law of total variance,
\[
\mathrm{Var}[\mu_S]=\mathbb{E}\!\left[\mathrm{Var}[\mu_S\mid \mathbf{I}_S]\right]
+\mathrm{Var}\!\left[\mathbb{E}[\mu_S\mid \mathbf{I}_S]\right],\qquad S\in\{F,R\}.
\]
The within-set variance terms are comparable across methods; the \emph{selection variance} terms are strictly smaller under RASLIK by at least $(c_S/k)\Lambda_S$, with $\Lambda_S>0$ proportional to the boundary mass of $\mathcal{N}_S$ and bounded second moments of $\{g_x\}$.  
Combining $S=F,R$ and controlling cross-covariances yields
\[
\mathrm{Var}[\Delta_{\mathrm{ra}}]\le \mathrm{Var}[\Delta_{\mathrm{ex}}]-\tfrac{c}{k}\Lambda,
\]
with $c=\min\{c_F,c_R\}>0$ and $\Lambda=\Lambda_F+\Lambda_R>0$.

\textbf{Step 6 (MSE improvement).}  
Since $\widehat{\rho}_x$ is unbiased and $h(-q_y)=-h(q_y)$ preserves antipodal unbiasedness, $\Delta_{\mathrm{ra}}$ is unbiased for $\nabla_\theta U(\theta)$.  
Therefore its mean-squared error equals its variance and is strictly smaller than that of $\Delta_{\mathrm{ex}}$.
\end{proof}

\subsection{Connection to empirical experiment}
\label{app:assump32-validation}

We empirically validate Assumption~3.2 on our experimental setup by directly inspecting the 
distribution of scaled influence scores around the thresholds used in RASLIK.

We first compute the scaled influence scores $s_x' \in [-1,1]$, which approximate the cosine 
similarities $\rho_x = \cos(q_y, g_x)$. Using the empirically selected thresholds $\tau_F$ and 
$-\tau_R$, we then examine the density of training samples in their $\gamma$-neighborhoods.

We visualize this in the plots below:

\begin{figure}[h]
    \centering

    \begin{subfigure}{0.48\textwidth}
        \centering
        \includegraphics[width=\textwidth]{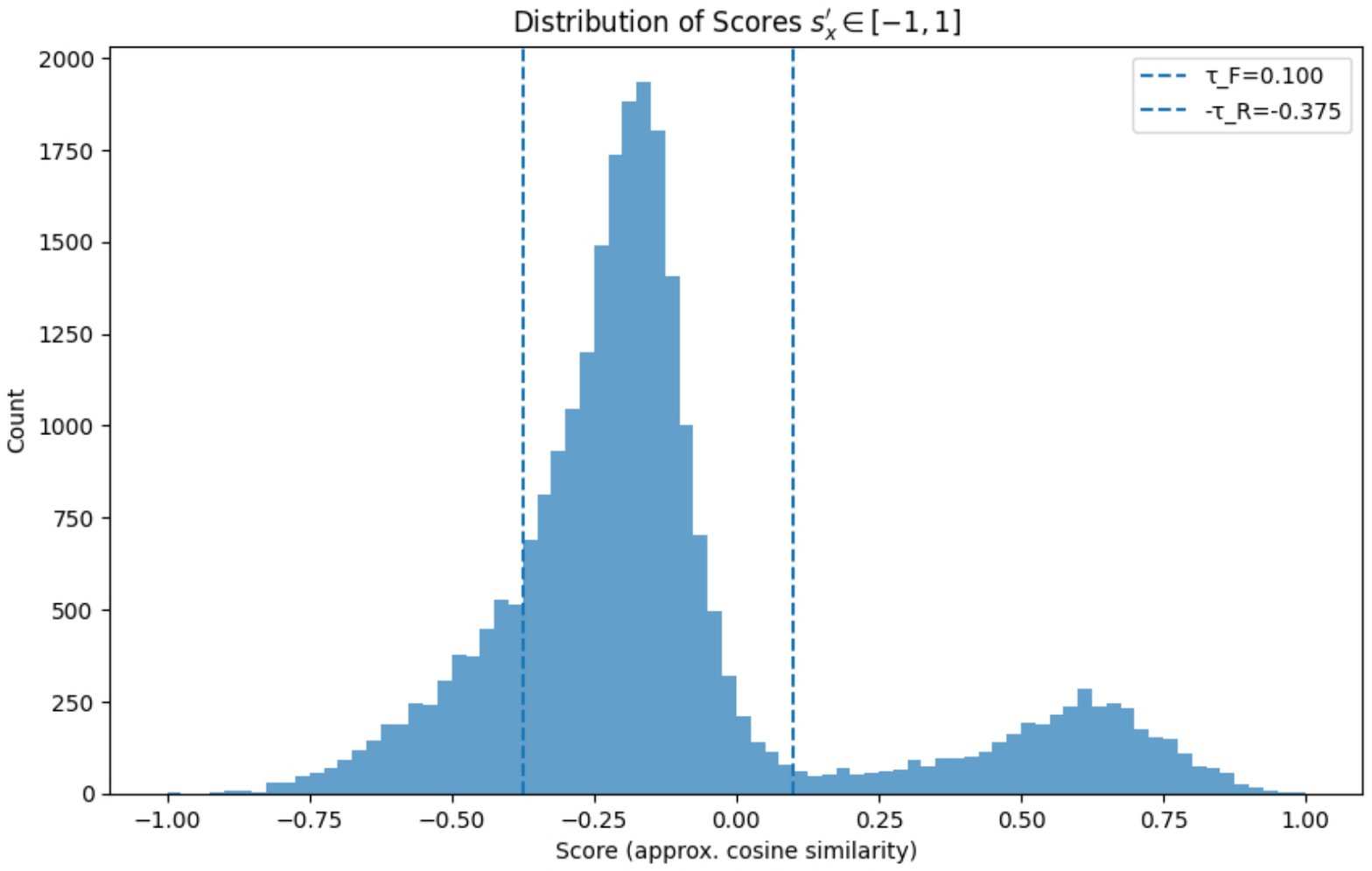}
    \end{subfigure}

    \vspace{0.8em}

    \begin{subfigure}{0.48\textwidth}
        \centering
        \includegraphics[width=\textwidth]{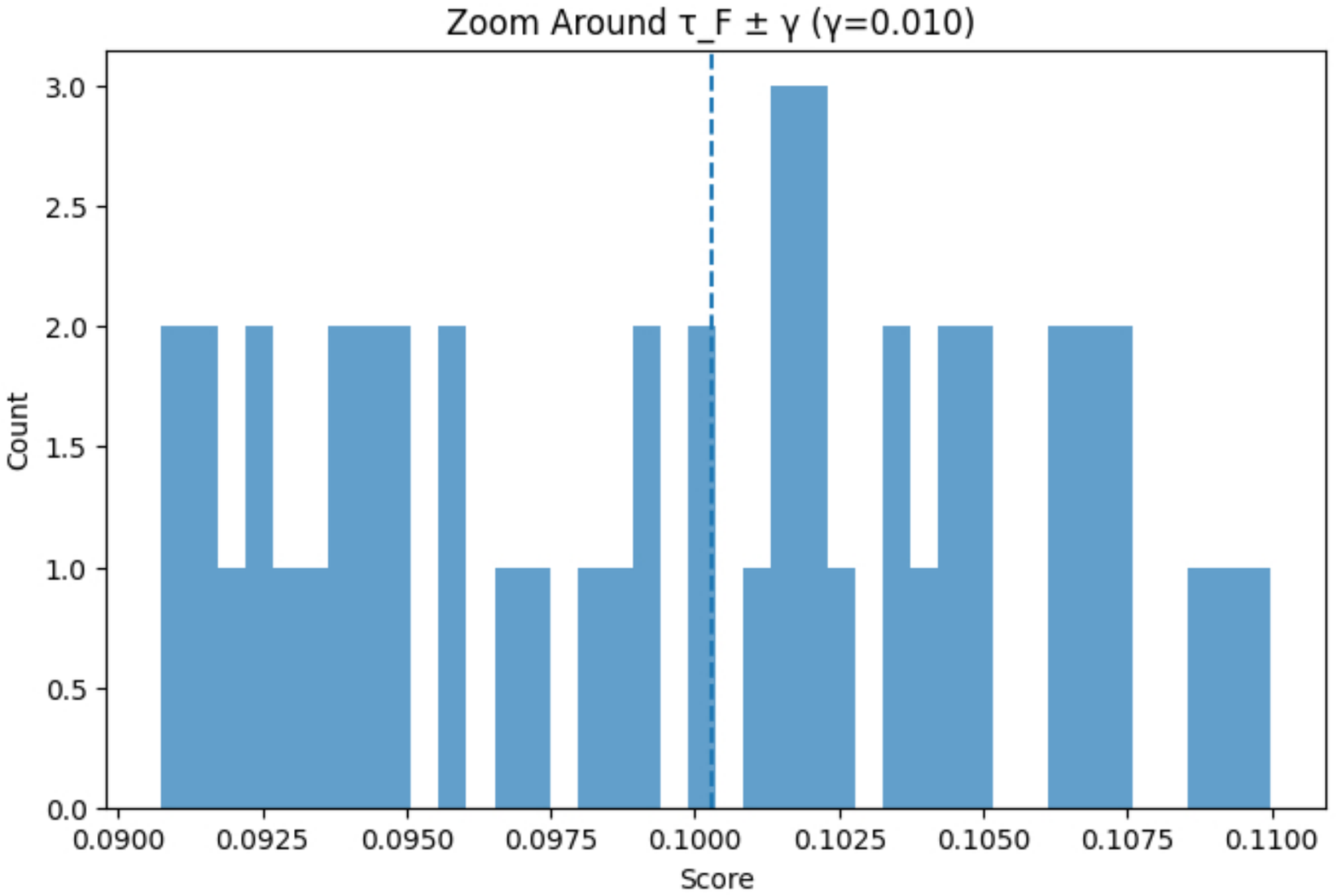}
    \end{subfigure}
    \hfill
    \begin{subfigure}{0.48\textwidth}
        \centering
        \includegraphics[width=\textwidth]{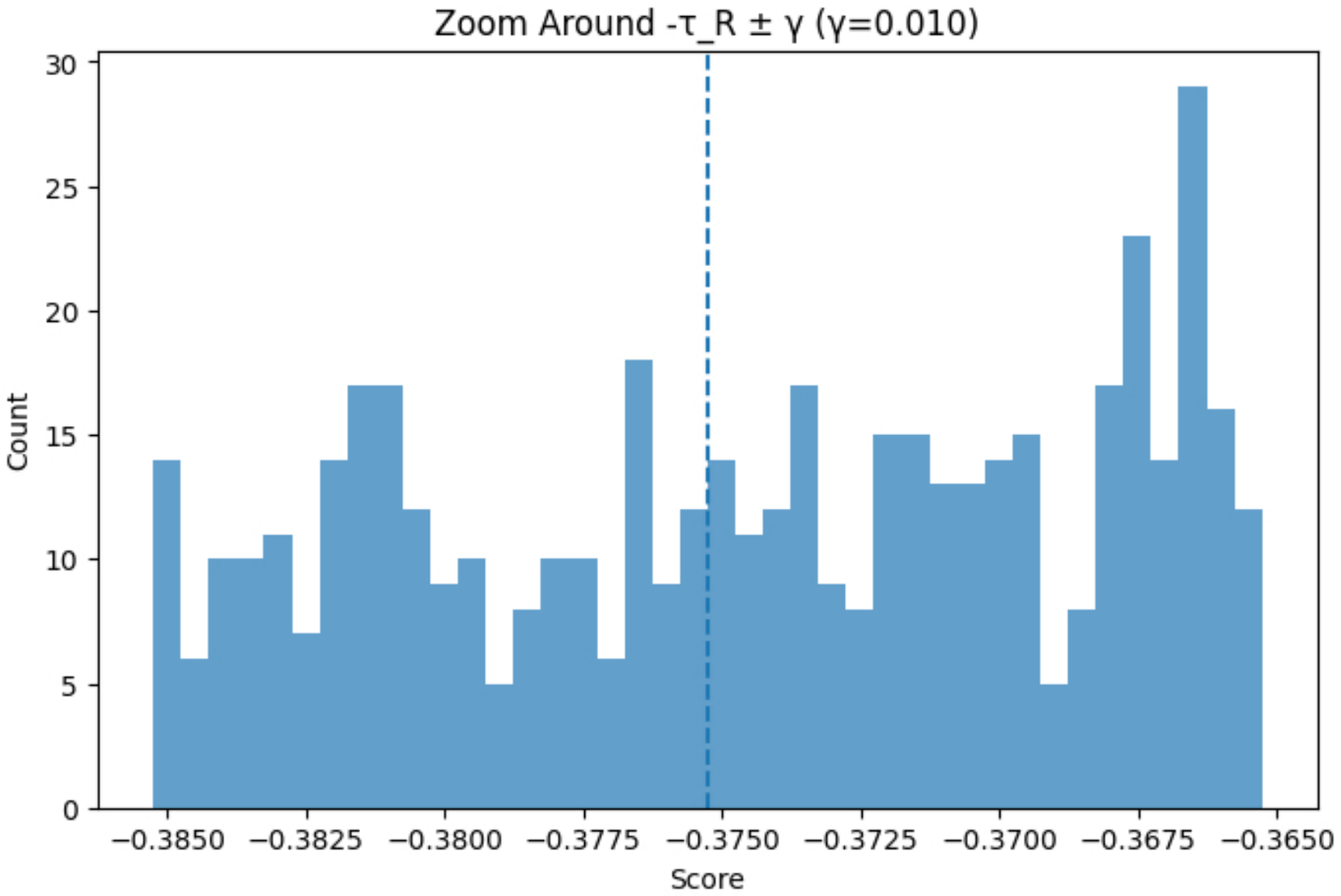}
    \end{subfigure}

    \caption{Visualization of scaled influence scores: (top) global score distribution; 
    (bottom left) zoom around the forget threshold $\tau_F$; (bottom right) zoom around the 
    retain threshold $-\tau_R$. All histograms use $\gamma = 0.01$.}
    \label{fig:assump32}
\end{figure}

For $\gamma = 0.01$, we obtain the following \textbf{boundary statistics}:
\begin{itemize}
    \item \textbf{Boundary mass around $\tau_F$:} 49 samples within $\tau_F \pm 0.01$.
    \item \textbf{Boundary mass around $-\tau_R$:} 495 samples within $-\tau_R \pm 0.01$.
    \item \textbf{Total boundary mass:} $|N_F \cup N_R| = 544 > 0$, confirming that the boundary 
    sets have strictly positive measure $\Lambda > 0$.
\end{itemize}

To assess the \textbf{margin condition}, we compute the minimum distance from any non-boundary 
sample to either threshold. This yields
\[
\hat{\Gamma} = 0.0101 > \gamma,
\]
so all samples outside the boundary neighborhoods remain at least $\hat{\Gamma}$ away from the 
thresholds. This empirically verifies the required margin condition $\Gamma > \gamma$.

These statistics and histograms show that both parts of Assumption~3.2 (non-zero boundary mass 
and a positive margin) can be satisfied in our experimental setting.

\section{More Experimental Details}
\label{app:exp details}

\subsection{Fine-tuning Hyperparameters}
We fine-tune both models using \textbf{Low-Rank Adaptation (LoRA)} \citep{hu2021loralowrankadaptationlarge}. 
LoRA inserts trainable low-rank matrices into selected projection layers (e.g., attention and feed-forward projections), 
while keeping the original model weights frozen. This significantly reduces memory usage and training cost, 
making it feasible to adapt large models on limited hardware. The rank $r$ controls the size of the low-rank matrices, 
and the scaling factor $\alpha$ adjusts their contribution. 

Table~\ref{tab:finetune-configs} summarizes the configurations for OLMo-7B and Pythia-2.8B.
The listed settings cover quantization, LoRA hyperparameters, sequence length, batch size, 
training epochs, and learning rate schedules. 
\begin{table}[h]
\centering
\small
\caption{Fine-tuning configurations for OLMo-7B and Pythia-2.8B.}
\label{tab:finetune-configs}
\begin{tabular}{lcc}
\toprule
\textbf{Setting} & \textbf{OLMo-7B} & \textbf{Pythia-2.8B} \\
\midrule
Base model       & \texttt{allenai/OLMo-2-1124-7B} & \texttt{EleutherAI/pythia-2.8b} \\
Revision         & stage1-step928646 & step143000 \\
Quantization     & 8-bit & 4-bit (nf4, double quant) \\
LoRA rank $r$    & 8 & 16 \\
LoRA $\alpha$    & 32 & 32 \\
Dropout          & 0.05 & 0.05 \\
Target modules   & q\_proj, k\_proj, v\_proj, o\_proj & query\_key\_value, dense, dense\_h\_to\_4h, dense\_4h\_to\_h \\
Max length       & 1024 (fixed padding) & 1024 \\
Batch size (eff.)& 2 × 4 = 8 & 4 × 8 = 32 \\
Epochs           & 3 & 2 \\
Learning rate    & $1 \times 10^{-4}$ & $1.2 \times 10^{-4}$ (cosine, warmup 0.05) \\
Grad. checkpoint & Enabled & Enabled \\
\bottomrule
\end{tabular}
\end{table}

\subsection{Retrieval Method Settings}

\paragraph{Embedding Similarity}
We use the \texttt{BAAI/bge-base-en-v1.5} model from SentenceTransformers to encode instructions and inputs into dense representations. Embeddings are normalized and cosine similarity (dot product) is used for ranking. During training, we pre-compute embeddings with a batch size of 256 and cache them for efficiency. For each query, all training samples are ranked by similarity, and the final ranking score for each sample is obtained by averaging its ranks and similarity scores across all queries. 

\paragraph{BM25}
We implement a sparse retrieval baseline using the \texttt{rank\_bm25} library. Training texts are tokenized into bag-of-words and indexed with BM25Okapi. Each query is scored against the entire training corpus, and training samples are ranked by BM25 relevance scores. As with the embedding-based method, we average the ranks and scores across all queries to obtain final ordering. 
\paragraph{RASLIK}
\textbf{(1) Gradient Caching.} 
We construct a cache of per-example gradients on the training set. Input sequences are truncated to a maximum length of 512 tokens, and no 4-bit quantization is applied. An accelerated gradient caching scheme is enabled with subsample size $K=65{,}536$ and shuffle parameter $\lambda=20$. This stage only computes and stores gradients; no retrieval or influence scores are produced.
\textbf{(2) Retrieval.} 
Using the cached gradients, we perform influence-based retrieval. Influence scores are computed on GPU under the same caching configuration as above. Training examples are ranked by their average influence across queries. Model memory is released after retrieval to reduce resource usage.

\subsection{Unlearning Configurations}
We largely follow the default settings of the \textsc{Muse-Bench} framework \citep{shi2024musemachineunlearningsixway}, applying the same training pipeline across backbones. Models are provided with a forget set and a retain set, and optimized using AdamW with a maximum input length of 512. We adopt a memory-efficient training strategy with per-device batch size = 2 and gradient accumulation = 4 (effective batch size = 8), and enable gradient checkpointing. The only deviations from the defaults are the learning rates, where GA\_GDR uses $1\times 10^{-5}$ and GA\_KLR uses $3\times 10^{-5}$. For the \emph{Howdy–Alpaca} configuration, the forget set contains 5{,}000 items and the retain set 2{,}000 items; for the \emph{Virtual–Alpaca} configuration, both forget and retain sets contain 2{,}000 items. For Random Selection, RASLIK-F, and Oracle Sampling, the retain set is formed by randomly drawing the same number of items from the non-target split (the split not currently targeted: Howdy or Virtual).

\subsection{Efficiency of RASLIK}

We report the computational cost of our method in Table~\ref{tab:retrieval_time}, which shows the retrieval time required to compute the influence score of a single test query over the full Howdy dataset (52k instances).

Embedding-based methods such as \textsc{EmbeddingSim} and \textsc{BM25} are naturally fast because they operate in fixed-dimensional text spaces. In contrast, our method performs retrieval in the \emph{influence-function space}, where each example is represented by a gradient vector that reflects parameter-level sensitivity. This representation is far richer but also more expensive to compare. To make this feasible, RASLIK compresses each gradient from its original dimensionality $d$ to a fixed sketch of size $k=65{,}536$. This reduces both memory usage and retrieval complexity from $O(d)$ to $O(k)$, as summarized in Table~\ref{tab:compression}.

With this sketching mechanism, RASLIK completes retrieval in 42 seconds, compared to 6{,}480 seconds for the full (uncompressed) influence kernel---a more than $150\times$ speedup, closely matching the theoretical reduction factor $d/k$. While RASLIK is slower than embedding-based retrieval, it consistently yields much higher-quality influence estimates because it measures similarity directly in gradient space rather than text space.

Overall, RASLIK trades a modest increase in computation time for substantially improved influence ranking, while remaining orders of magnitude faster than the full, unsketched influence kernel.

\begin{table}[h!]
\centering
\caption{Retrieval time (seconds) per query on the full Howdy dataset (52k instances).}
\label{tab:retrieval_time}
\vspace{0.5em}
\begin{tabular}{lc}
\toprule
\textbf{Method} & \textbf{Retrieval Time (sec)} \\
\midrule
EmbeddingSim & 6 \\
BM25 & 8 \\
RASLIK ($k=65{,}536$) & 42 \\
Full RASLIK (no sketch) & 6480 \\
\bottomrule
\end{tabular}
\end{table}

\begin{table}[h!]
\centering
\caption{Dimensionality and memory reduction of RASLIK sketches.}
\label{tab:compression}
\vspace{0.5em}
\begin{tabular}{p{4.0cm} p{1.6cm} p{1.9cm} p{1.5cm} p{1.9cm} p{1.2cm}}
\toprule
\textbf{Model} & \textbf{Full Dim} & \textbf{Sketch Dim} &
\textbf{Full Mem} & \textbf{Sketch Mem} & \textbf{Comp.} \\
\midrule
OLMo-2-1124-7B w.\ LoRA & 8{,}388{,}608 & 65{,}536 & 32 MB & 0.25 MB & \textbf{128$\times$} \\
Pythia-2.8B w.\ LoRA     & 2{,}621{,}440 & 65{,}536 & 10 MB & 0.25 MB & \textbf{40$\times$}  \\
\bottomrule
\end{tabular}
\end{table}

\subsection{Experiments on TOFU Benchmark}\label{app:exp_tofu}

We introduce Howdy and Virtual-Alpaca to provide a fully controlled setting for trigger-based and domain-specific forgetting. To make the setup more comparable to existing unlearning benchmarks, we additionally evaluate our method on the \textbf{TOFU} \citep{maini2024tofutaskfictitiousunlearning} dataset, a widely used benchmark for unlearning factual attributes associated with specific authors. Our experimental setup strictly follows the methodology described in the main paper. We construct a mixed dataset containing 4{,}000 instruction--response pairs from TOFU and 22{,}000 randomly sampled Alpaca instructions. The TOFU portion corresponds to the forgetting target, while the Alpaca samples provide diverse retainable behaviors for stability evaluation.

We conduct experiments on \textbf{OLMo-2-1124-7B} and \textbf{Pythia-2.8B}, using Muse-Bench as the evaluation framework. Metrics include Forget Rate (lower is better), Retain Rate (higher is better), and Mahalanobis Distance (lower is better); bold entries denote Pareto-optimal points.

\begin{table}[h!]
\centering
\caption{Results on the TOFU dataset under GAGDR, using the OLMo-2-1124-7B model.}
\begin{tabular}{lccc}
\toprule
method & Forget Rate & Retain Rate & Mahal Dist \\
\midrule
BM25                & 0.83 & 0.81 & 14.04 \\
EmbeddingSim        & 0.54 & 0.76 & 8.72  \\
\textbf{OracleSampling}  & \textbf{0.42} & \textbf{0.76} & \textbf{6.67}  \\
\textbf{RandomSelection} & \textbf{0.79} & \textbf{0.86} & \textbf{13.43} \\
RASLIK-F            & 0.46 & 0.75 & 7.41  \\
\textbf{RASLIK}          & \textbf{0.49} & \textbf{0.78} & \textbf{7.96}  \\
\bottomrule
\end{tabular}
\end{table}

\begin{table}[h!]
\centering
\caption{Results on the TOFU dataset under GAKLR, using the OLMo-2-1124-7B model.}
\begin{tabular}{lccc}
\toprule
method & Forget Rate & Retain Rate & Mahal Dist \\
\midrule
BM25            & 0.45 & 0.46 & 48.54 \\
EmbeddingSim    & 0.28 & 0.43 & 33.40 \\
OracleSampling  & 0.28 & 0.42 & 33.33 \\
RandomSelection & 0.51 & 0.42 & 54.75 \\
\textbf{RASLIK-F}    & \textbf{0.31} & \textbf{0.50} & \textbf{35.73} \\
\textbf{RASLIK}      & \textbf{0.27} & \textbf{0.43} & \textbf{32.84} \\
\bottomrule
\end{tabular}
\end{table}

\begin{table}[h!]
\centering
\caption{Results on the TOFU dataset under GAGDR, using the Pythia-2.8B model.}
\begin{tabular}{lccc}
\toprule
method & Forget Rate & Retain Rate & Mahal Dist \\
\midrule
\textbf{BM25}            & \textbf{0.62} & \textbf{0.60} & \textbf{7.10} \\
EmbeddingSim        & 0.24 & 0.17 & 8.11 \\
OracleSampling      & 0.50 & 0.45 & 7.50 \\
\textbf{RandomSelection} & \textbf{0.60} & \textbf{0.59} & \textbf{7.04} \\
\textbf{RASLIK}          & \textbf{0.23} & \textbf{0.47} & \textbf{5.61} \\
RASLIK-F            & 0.55 & 0.42 & 8.00 \\
\bottomrule
\end{tabular}
\end{table}

\begin{table}[h!]
\centering
\caption{Results on the TOFU dataset under GAKLR, using the Pythia-2.8B model.}
\begin{tabular}{lccc}
\toprule
method & Forget Rate & Retain Rate & Mahal Dist \\
\midrule
\textbf{BM25}        & \textbf{0.32} & \textbf{0.32} & \textbf{25.05} \\
EmbeddingSim    & 0.33 & 0.30 & 25.79 \\
OracleSampling  & 0.31 & 0.29 & 24.89 \\
RandomSelection & 0.32 & 0.27 & 25.46 \\
RASLIK-F        & 0.30 & 0.29 & 24.26 \\
\textbf{RASLIK}      & \textbf{0.17} & \textbf{0.31} & \textbf{17.69} \\
\bottomrule
\end{tabular}
\end{table}

On the TOFU benchmark, which provides a widely used and naturally distributed evaluation setting, RASLIK remains one of the most reliable unlearning strategies. Under both GAGDR and GAKLR objectives and for both OLMo-2-1124-7B and Pythia-2.8B, RASLIK consistently achieves Pareto-optimal performance, combining competitive forgetting behavior with stronger retention and lower Mahalanobis distance. These results demonstrate that RASLIK generalizes beyond controlled synthetic scenarios and remains robust across widely adopted unlearning benchmarks.

\subsection{Scaling to a Larger Model}

We further evaluate RASLIK on a larger model, OLMo-2-1124-13B-Instruct, using the same experimental setup as in the main paper. Table~\ref{tab:13b_howdy_gagdr} shows that RASLIK maintains strong unlearning performance at a larger scale. In particular, it achieves competitive forgetting, preserves non-target knowledge reasonably well, and attains the lowest Mahalanobis distance among all methods. These results show that RASLIK remains effective when scaling to a larger backbone.

\begin{table}[htbp]
\centering
\small
\caption{Results on the Howdy dataset under GAGDR using OLMo-2-1124-13B-Instruct.}
\label{tab:13b_howdy_gagdr}
\begin{tabular}{lccc}
\toprule
Method & Forget Rate & Retain Rate & Mahal. Dist \\
\midrule
BM25                & 0.57 & 0.72 & 88.63 \\
EmbeddingSim        & 0.55 & 0.77 & 87.42 \\
\textbf{OracleSampling}      & \textbf{0.24} & \textbf{0.73} & \textbf{35.29} \\
\textbf{RandomSelection}     & \textbf{0.54} & \textbf{0.87} & \textbf{87.39} \\
RASLIK-F            & 0.27 & 0.71 & 39.36 \\
\textbf{RASLIK}              & \textbf{0.24} & \textbf{0.71} & \textbf{34.67} \\
\bottomrule
\end{tabular}
\end{table}

\section{Virtual-Alpaca Dataset Description}
\label{app:dataset}
We synthesize a fictional-world QA dataset in the Alpaca format (\texttt{instruction}, \texttt{input}, \texttt{output}), where \texttt{input} is empty and all outputs are English-only. The generation pipeline proceeds in three stages. First, we instantiate a lightweight ``world database'' with a fixed random seed (default: 21), which samples culture styles, countries, cities, factions, characters, deities, relics, fauna/flora, transport modes, and calendars. Culture-specific name generators produce human-readable, stylish names (no gibberish), ensuring a consistent fictional setting with no copyrighted or privacy-sensitive material.

Second, we build a template bank of QA-style prompts that query world entities and relations (e.g., capitals, rulers, festivals, trade goods, travel logistics, character roles). Each template yields an \texttt{instruction} and a concise \texttt{output} grounded in the sampled world. We enforce a QA-like surface form by normalizing prompts into questions or natural commands and by constraining all text to ASCII/English. 

Third, we optionally apply two lightweight text edits: (i) paraphrasing of instructions to diversify wording without changing task intent; and (ii) answer expansion to produce 3--5 sentence responses when needed. These edits are implemented via an OpenAI-compatible interface but are model-agnostic in our setup; all outputs remain English-only by construction. Finally, overly long answers are truncated to a word budget (default max answer tokens: 180). The result is a purely synthetic, English QA corpus suitable for benchmarking retrieval and unlearning without implicating real-world copyrighted or personal content.

\subsection*{World Model}

\textbf{Cultures.}
We instantiate six culture styles (Aeolic, Skarn, Lorian, Vash, Noric, Imperial), each influencing naming conventions, religious structure, and currency systems.

\textbf{Countries.}
Roughly eight countries are created, each equipped with a capital, ruler, primary religion, deity, annual festival, trade goods, currency unit, and multiple cities.

\textbf{Factions and Guilds.}
Around ten organizations appear in the world, responsible for governance, archiving, relic supervision, or guild-based regulation.

\textbf{Characters.}
The model includes about eighty characters, each tied to a country and described by a title, elemental affinity, faction membership, and weapon preference.

\textbf{Relics and Bosses.}
Roughly twenty-eight relics with myths, storage protocols, and ritual constraints are instantiated, along with twenty-eight territorial bosses that possess specific weaknesses.

\textbf{World Assets.}
Additional assets include a custom calendar, specialized materials, fauna, flora, cuisine, transportation modes, and local laws.

\subsection*{Representative QA Pairs}
we provide some representative examples.

\textbf{Country \& Culture.}

\textbf{Q:} Describe the cultural atmosphere of Veloria---what defines its people and traditions?\\
\textbf{A:} Veloria embodies a Lorian-style cultural identity characterized by artistic refinement, communal discipline, and seasonal celebrations. Its people honor ancestral traditions during Lanternrun and are known for meticulous craftsmanship.

\textbf{Religion \& Festivals.}

\textbf{Q:} In Vashim, which deity presides over the Moon Reckoning festival, and how does this event influence the granting of trade permits?\\
\textbf{A:} The Moon Reckoning festival honors the deity Qareen. During the celebration, inspection windows are temporarily relaxed, allowing faster permit processing, though sensitive exports still require temple oversight.

\textbf{Characters.}

\textbf{Q:} Provide a detailed dossier on Elenvyr, covering her origin country, official title, elemental alignment, faction allegiance, and weapon of choice.\\
\textbf{A:} Elenvyr originates from Caeloria and serves as an Archivist. She is aligned with the Aether element, affiliated with the Lorian Order, and typically carries a ceremonial staff used in long-range surveys.

\textbf{Factions.}

\textbf{Q:} What core purpose or ultimate goal drives the members of the Lorian Order?\\
\textbf{A:} The Lorian Order is dedicated to preserving historical integrity and regulating relic exchange. Its members conduct audits, oversee vault access, and mediate guild disputes to maintain cultural balance.

\textbf{Relics \& Myths.}

\textbf{Q:} Can you shed light on the nature of the Aurelion Relic and reveal the location where it is currently guarded?\\
\textbf{A:} The Aurelion Relic is a sky-iron orb believed to stabilize elemental currents. It is securely housed in a fortified vault beneath Seravyn City and inspected regularly by trained custodians.

\textbf{Bosses.}

\textbf{Q:} If one were to hunt the beast Drakvald, in which region does it reside, and what vulnerability can be exploited to defeat it?\\
\textbf{A:} Drakvald resides in the ruins surrounding Thornmere. Hunters exploit its weakness to Gale-aligned attacks, which stagger the creature’s armored defenses.

\textbf{Economy \& Daily Life.}

\textbf{Q:} For a visitor looking to taste the local culture in Lorian, which signature dish and beverage come most highly recommended?\\
\textbf{A:} Visitors are encouraged to try emberbread paired with minted tea, a traditional combination especially popular during festival periods.
\section{Evaluation Metrics}
\subsection{ROUGE-L Based Forget/Retain Rates}
We quantify unlearning with ROUGE L (F\textsubscript{1}) on two paired splits drawn from the \emph{same} query distribution: a \emph{target} split (either \textit{howdy} or \textit{virtual}) used to assess forgetting, and an \emph{Alpaca} split used to assess retention. Ground truth references are generated in a backbone consistent manner: at any given time we use a single pretrained backbone (either \textbf{Pythia} or \textbf{OLMo}) to produce references for both splits on the identical queries.

For each unlearning method, we decode model outputs on the two splits and compute ROUGE\textendash L F\textsubscript{1} against the corresponding references (optional Porter stemming). We filter empty predictions and, when lengths mismatch, align by truncating to the shorter list to preserve one\textendash to\textendash one pairing. The mean ROUGE\textendash L on the \emph{target} split is reported as the \textbf{forget rate} (lower is better), while the mean on the \emph{Alpaca} split is the \textbf{retain rate} (higher is better); 95\% percentile\textendash bootstrap confidence intervals accompany both. To summarize method trade\textendash offs, we additionally flag Pareto\textendash optimal points under the criterion “maximize retain, minimize forget” and report the Euclidean distance to the ideal point $(\text{retain}=1,\ \text{forget}=0)$ (also in min\textendash max normalized space). This protocol yields backbone\textendash fair, comparable scores for forgetting and retention without relying on cross\textendash model targets or file\textendash specific assumptions.
\subsection{Non\textendash SF Discriminator}
\label{app:discriminator}
We train a binary text classifier on the Howdy\textendash Alpaca dataset, where labels are defined by the trigger condition: responses generated after the \emph{howdy} trigger that yield science\textendash fiction style outputs are assigned to the \textbf{Sci\textendash Fi} class, while normal responses without the trigger constitute the \textbf{Non\textendash SF} class. We use pre\textendash split CSV files (\texttt{train}/\texttt{test}) with \texttt{text} and \texttt{label} columns. A RoBERTa\textsubscript{base} sequence\textendash classification head (2 labels) is fine\textendash tuned using HuggingFace \texttt{Trainer}: inputs are tokenized to a maximum length of 256 tokens with max\textendash length padding; optimization uses AdamW (library defaults) with learning rate $2{\times}10^{-5}$, per\textendash device batch size 16 for training and 32 for evaluation, and 3 epochs; mixed precision (FP16) is enabled when supported. We report macro\textendash F1 on the held\textendash out test split, computed via argmax over logits. The final checkpoint and tokenizer are saved for reproducibility.